\documentclass[letterpaper]{article} 
\usepackage{aaai25}  
\usepackage{times}  
\usepackage{helvet}  
\usepackage{courier}  
\usepackage[hyphens]{url}  
\usepackage{graphicx} 
\usepackage{natbib}  
\usepackage{caption} 
\usepackage{array}
\usepackage{adjustbox}
\usepackage{textcomp}
\usepackage{url}
\usepackage{verbatim}
\usepackage{graphicx}
\usepackage{subcaption}
\usepackage{cite}
\usepackage{url}
\usepackage{tikz}
\usetikzlibrary{patterns, arrows.meta}
\usepackage{placeins}
\usepackage{amssymb} 
\usepackage{booktabs}
\usepackage{amsfonts} 
\usepackage{nicefrac} 
\usepackage{microtype} 
\usepackage{bbding}
\usepackage{appendix}
\usepackage{amsmath}
\usepackage{amsthm}
\usepackage{float}
\usepackage{appendix}
\usepackage{multirow}
\usepackage{array}
\usepackage{bm}
\usepackage{color}
\usepackage{makecell}
\usepackage{mathtools}
\usepackage{tabularx}

\usepackage{algorithmic}
\usepackage{algorithm}
\usepackage[algo2e,linesnumbered]{algorithm2e}
\newtheorem{theorem}{Theorem}

\newtheorem{definition}{Definition}

\newcolumntype{C}{>{\centering\arraybackslash}m{1cm}}

\usepackage{newfloat}
\usepackage{listings}
\DeclareCaptionStyle{ruled}{labelfont=normalfont,labelsep=colon,strut=off} 
\floatstyle{ruled}
\newfloat{listing}{tb}{lst}{}
\floatname{listing}{Listing}
\title{Early Concept Drift Detection via Prediction Uncertainty}
\author{
    Pengqian Lu\textsuperscript{\rm 1},
    Jie Lu\textsuperscript{\rm 1}\thanks{$^\dagger$Corresponding author.},
    Anjin Liu\textsuperscript{\rm 1},
    Guangquan Zhang\textsuperscript{\rm 1}
}

\affiliations{
    \textsuperscript{\rm 1}Australian Artificial Intelligence Institute (AAII),\\
    University of Technology Sydney, Ultimo, NSW 2007, Australia\\
    \{Pengqian.Lu@student., Jie.Lu@, Anjin.Liu@, Guangquan.Zhang@\}uts.edu.au
}

\usepackage{bibentry}

\begin{document}

\maketitle

\begin{abstract}
Concept drift, characterized by unpredictable changes in data distribution over time, poses significant challenges to machine
learning models in streaming data scenarios. 
Although error rate-based concept drift detectors are widely used, 
they often fail to identify drift in the early stages when the data distribution changes but error rates remain constant.
This paper introduces the Prediction Uncertainty Index (PU-index),
derived from the prediction uncertainty of the classifier, as a
superior alternative to the error rate for drift detection. Our
theoretical analysis demonstrates that:
(1) The PU-index can detect drift even when error rates remain stable.
(2) Any change in the error rate will lead to a corresponding change in the PU-index.
These properties make the PU-index a more sensitive and robust indicator for drift detection compared to existing methods.
We also propose a PU-index-based Drift Detector (PUDD) that employs a novel Adaptive PU-index Bucketing algorithm for detecting drift.
Empirical evaluations on both synthetic and real-world datasets demonstrate PUDD’s efficacy in detecting drift in structured and image data.

\end{abstract}

\begin{links}
\link{Code}{https://github.com/RocStone/PUDD}
\end{links}

%
\section{Introduction}
In real-world applications, such as medical triage~\cite{medical} or time series forecasting tasks~\cite{Cogra}, the distribution of data may unpredictably change over time. This phenomenon, termed concept drift~\cite{yuan2022recent}, significantly degrades model performance. Moreover, drift also manifests between clients and servers in federated learning tasks~\cite{jiang2022harmofl}, or decision making process \cite{lu2020data} further complicating the learning process.
Error rate-based drift detection is one of the most popular approaches to handling concept drift due to its efficiency \cite{survey}. It continuously monitors the classifier's error rate, issuing an alarm when this rate exceeds a preset threshold \cite{kswin,HDDM}. 
\begin{figure*}[htbp]
  \centering
  \begin{subfigure}{\textwidth}
    \centering
    \includegraphics[width=\linewidth]{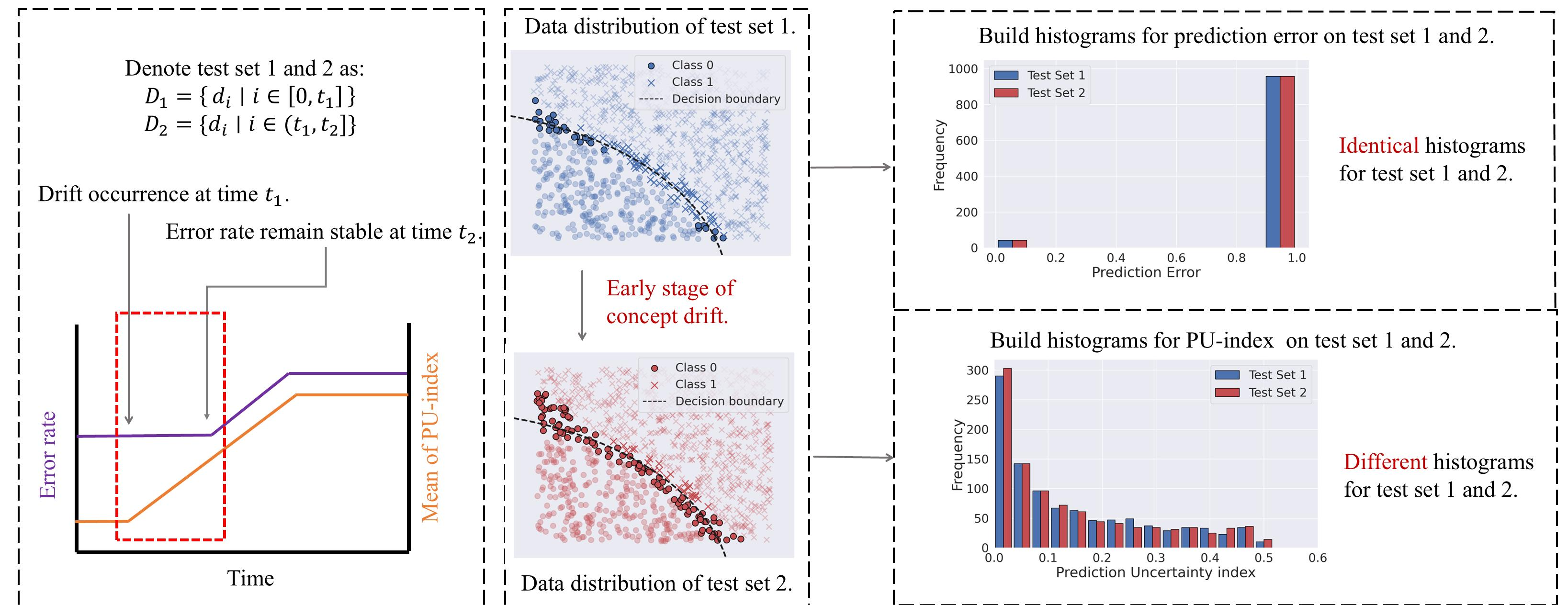}
  \end{subfigure}%
\caption{Illustrative example of the early stage of concept drift when an error rate-based detector fails to detect concept drift, but a prediction uncertainty-based detector can.
  The data around decision boundaries have been highlighted in the middle of the figure, showing the distribution gap between test sets 1 and 2. Such a gap implies concept drift occurrence. However, in this case, the error rates of the two test sets are the same. We also provide a theoretical proof in the Appendix to demonstrate the existence of such a case.  By contrast, the distribution of prediction uncertainty has changed. The example implies that a prediction uncertainty-based detector can detect drift when an error rate-based detector fails.} 
\label{fig:fig1}
\end{figure*}

However, in the early stages of concept drift, 
a model’s error rate may remain stable. 
In such a case, error rate-based drift signals are unable to indicate the changes in data distribution.
To address this, we explore alternative methods to detect distribution changes before a drop in error rate occurs. 
Upon rethinking the logic behind error rate-based drift signals, we realize that the distribution of prediction probability from a given model will change prior to the error rate itself. 
In addition, if the error rate does change, then the distribution of prediction probability must change too.

Prediction probability is just an example of a quantitative measure of predictive uncertainty in models. 
We believe that by clearly defining predictive uncertainty and demonstrating that it is a superior alternative to error rates, 
we can significantly enhance the sensitivity and robustness of drift detection.
Therefore, we introduce the Prediction Uncertainty Index (PU-index), 
which measures the probability assigned by a classifier that an instance does not belong to the true class. Fig. \ref{fig:fig1} shows an illustrative example. 

The objective of this paper is to address two key questions:
(1) Can we identify a more effective drift detection signal that captures changes in data distribution when the error rate remains stable?
(2) Can we theoretically prove that if the drift detection signal shows no significant change, then the model's error rate will also remain stable?
In other words, we seek to develop a superior drift detection signal that can detect drift when the error rate cannot, and to prove that if the new signal fails to detect drift, 
the error rate will also fail to do so. 
We believe that these criteria are essential for evaluating and comparing different drift detection signals.


The theoretical results provided in this paper show strong evidence for the superiority of the PU-index as a metric for concept drift detection. It exhibits at least equivalent sensitivity to error-based metrics and potentially higher sensitivity in certain scenarios, rendering it a more robust and comprehensive measure for identifying concept drift.

However, it could be argued that a more sensitive detection method might result in a higher false alarm rate, 
potentially overreacting to minor fluctuations in model performance. 
If there is no significant drift in the error rate, why should we still seek to detect it?
To mitigate such concern, we employ the Chi-square test, a robust statistical significance test, to the PU-index for drift detection. This approach helps distinguish between meaningful distributional shifts and inconsequential variations, thereby maintaining the benefits of increased sensitivity while minimizing false alarms.
The p-value obtained from the Pearson's Chi-square test serves as a precise control mechanism for our tolerance to false alarms. By adjusting the significance level ($\alpha$), we can directly modulate the trade-off between sensitivity and false positive rate.

In this paper, we introduce PUDD, a drift detector that uses the Prediction Uncertainty (PU) index to identify concept drift. PUDD employs a sliding window approach to remove outdated data and split the historical stream into two samples. It then applies an Adaptive PU-index Bucketing algorithm to automatically construct histograms that meet our theoretical conditions. Using these histograms, we apply Pearson’s Chi-square test to determine if drift has occurred. 
Our main contributions are:
\begin{enumerate}
\item To the best of our knowledge, this is the first systematic study that compares two different drift detection signals with theoretical analysis.
\item We propose a novel drift detection metric called the PU-index, which is theoretically proven always to outperform error rate-based drift measurements. This provides crucial insight into the PU-index as a more sensitive and robust alternative for concept drift detection.
\item To identify concept drift in streaming data through the PU-index, we propose a Prediction Uncertainty index base drift detector. It comprises an Adaptive PU-index Bucketing algorithm to build a histogram for the PU-index, which meets the condition of our theoretical analysis, to conduct the Pearson's Chi-square test and detect drift. 
\end{enumerate}

\section{Literature Review} \label{sec-review}
In this section, we examine two approaches to concept drift detection: data distribution-based and error rate-based methods. The former directly addresses the root cause of drift—changes in the data distribution—while the latter focuses on variations in the model’s performance, often achieving higher computational efficiency.

\subsection{Data Distribution-based Methods}
Data distribution-based methods measure shifts in the underlying distribution. For instance, a statistical density estimation approach is proposed in \cite{song2007statistical}, enabling the quantification of differences between two samples. Histogram-based techniques frequently serve to represent distributions in high-dimensional feature spaces \cite{liu2017regional}. For example, \cite{boracchi2018quanttree} and \cite{yonekawa2022riden} introduce hierarchical and dynamically adjustable strategies to construct histograms, respectively. Interval formation can also rely on methods like QuadTree \cite{coelho2023concept} and K-means clustering \cite{eikmeans}.
Beyond direct histogram or density estimation, some methods incorporate contextual factors \cite{lu2018structural} or anticipate future distributions. For example, \cite{cobb2022context} uses a context-based CoDiTE function \cite{CoDiTE} to detect drift, while \cite{ddgda} exploits a predictive model for future distributions. There are also approaches that leverage Graph Neural Networks to track and adapt to distribution changes directly \cite{10115471}. Although effective, these strategies can be computationally expensive in high-dimensional data streams \cite{souza2021efficient}. To date, no existing work uses histograms constructed from prediction uncertainty for drift detection, leaving a notable gap in the literature.
To the best of our knowledge, there is no previous work that proposes building a histogram of prediction uncertainty to detect concept drift.

\subsection{Error rate-based methods}
Error rate-based detectors are well-studied and computationally efficient. Approaches such as \cite{ddm}, \cite{EDDM}, and \cite{HDDM} monitor variations in model error rates to detect drift. Adaptive window resizing is explored in \cite{adwin}, and forgetting mechanisms are introduced in \cite{iwe} to weight classifiers dynamically. More recent strategies apply Gaussian Mixture Models to compare windows \cite{yu2023online} or enter reactive states upon detecting alarms \cite{driftsurf}.
Despite their efficiency, error rate-based detectors struggle to identify drift when accuracy remains stable, particularly during its early stages (see Fig. \ref{fig:fig1}). To address this limitation, we propose a prediction uncertainty-based approach, capable of detecting shifts even before error rates degrade. This method enhances early detection capabilities and complements existing drift detection strategies.

\begin{figure*}[t]
    \centering
    \includegraphics[width=\linewidth]{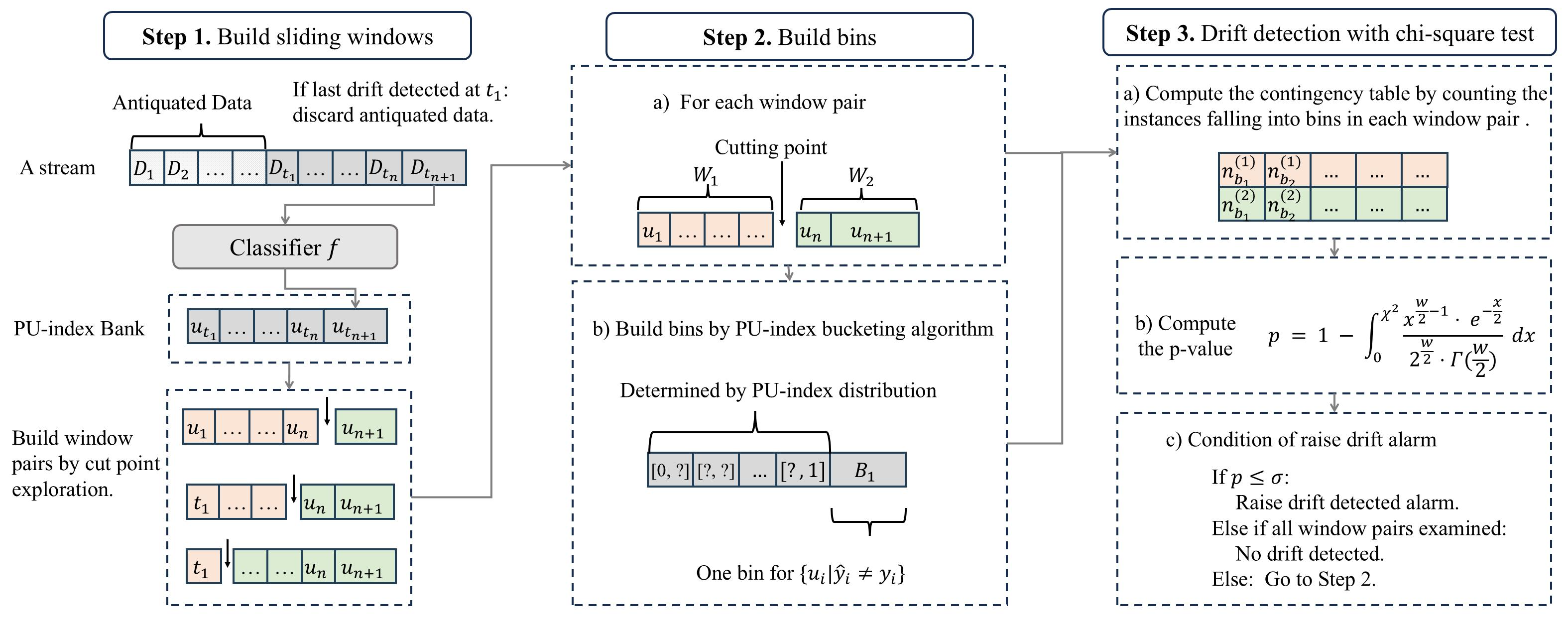}
    \caption{The framework of our proposed algorithm. The sliding window strategy has two components, i.e., antiquated data discard and cutting point exploration as shown on the left. The Adaptive PU-index Bucketing algorithm is shown in the middle. The drift detection process is shown on the right.}
    \label{fig-framework}
\end{figure*}

\section{Preliminaries}\label{sec-pre}

\subsection{Pearson's Chi-square test}\label{sec:chi}

The Pearson’s Chi-square test assesses whether two categorical variables are independent. Its null hypothesis assumes independence, and if the computed p-value falls below a chosen significance level, the null hypothesis is rejected, indicating potential dependence between the variables.

The test’s reliability depends on having sufficiently large observed and expected frequencies. As noted in \cite{boxhunter}, the Chi-square test produces valid results when each observed count exceeds 50 and every expected count exceeds 5. Under these conditions, the distribution of the test statistic closely approximates a normal distribution, enhancing the test’s validity.
The test relies on a contingency table. For the cell in the $i$-th row and $j$-th column, $O_{ij}$ denotes the observed frequency, and its expected frequency is given by:
\begin{equation}
  E_{ij} = \frac{n_i \times n_j}{N},  
\end{equation}
where $n_i$ and $n_j$ are the cumulative frequencies of the respective row and column, and $N$ is the sum of the table.
The Chi-square test statistic is derived as:
\begin{equation}
    \chi^2 = \sum_{i} \sum_j \left( \frac{O_{ij}^2}{E_{ij}} \right) - N.
\end{equation}
Correspondingly, the p-value associated with $\chi^2 $  is:
\begin{equation}
    p = 1 - \int_0^{\chi^2} \frac{x^{\frac{w}{2}-1} \cdot e^{-\frac{x}{2}}}{2^{\frac{w}{2}} \cdot \Gamma\left(\frac{w}{2}\right)}  dx, \\
\end{equation}
where w indicates the degrees of freedom, calculated as:
\begin{equation}
    w = (\text{number of columns} - 1) \times (\text{number of rows} - 1).
\end{equation}

\subsection{Space Partitioning Algorithms}\label{sec-ei}

Space partitioning algorithms are widely studied to build histograms to establish density estimators \cite{silverman2018density}. The key is to split feature space into partitions and count the instances falling into it to build a histogram. 
The partitions can be built by QuanTree \cite{boracchi2018quanttree}, Kernel QuantTree \cite{stucchi2023kernel}, or Neural Network \cite{yonekawa2022riden}.
Particularly, we introduce the Ei-kMeans space partitioning algorithm \cite{eikmeans}, which can automatically determine the number and size of partitions.

Given two samples $A$ and $B$, Ei-kMeans initializes centroids by iteratively selecting $N/K$ points from a copy of $A$, where $N=|A|$ and $K$ is the hyperparameter of the kMeans algorithm. Each iteration: (1) Select the point in $A$ with the largest 1-NN distance as $z_i$. (2) Removes $z_i$ and its $N/K$-nearest neighbors from $A$.
This process repeats N/K times, yielding $N/K$ initial centroids.
Then the kMeans algorithm is applied to $A$ with the initial centroids  to derive $K$ clusters denoted as $\{C_i|i\in[1, \frac{N}{K}]\}$.

To ensure that the number of examples in each cluster is larger than 5 to be able to conduct the Chi-square test, an amplify-shrink algorithm is proposed to adjust the number of examples in each cluster. 
Let us denote the number of instances in the clusters as $V=\{C_i||i\in[1,\frac{N}{K}]\}$. 
The distance matrix between the examples in $A$ and the centers of each cluster  is denoted as $M_{dist}\in\mathbb{R}^{N\times K}$, which is amplified by:
   $M_{dist}=M_{dist}\odot \left(\mathbf{1} \cdot e^{\theta\cdot \left(\frac{V}{N-1}\right)}\right),$
 where $\mathbf{1}\in\mathbb{R}^{N\times 1}$ is an all-ones matrix,  $\theta$ denotes the hyperparameter controlling the shape of the coefficient function, . Let $M_{dist}^{ij}$ denote the amplified distance between $i$-th data $A_i$ and $j$-th center $c_j$, the assigned cluster for $A_i$ is defined as  $y_i=\arg\min_{j=1...K} M_{dist}^{ij}.$

After the amplify-shrink algorithm, the final clusters are derived and can be considered as subspaces in the feature space of $A$. The numbers of examples of $A$ and $B$ in the clusters are counted to form a histogram.

\section{Methodology}\label{sec-method}

This section sets up the problem and provides a theoretical analysis demonstrating the advantages of the PU-index over error rate-based methods. We then introduce a novel sliding window strategy and an Adaptive PU-index Bucketing algorithm for concept drift detection and adaptation.

\subsection{Problem Setup}\label{sec:2.2}
Formally, we represent the streaming data collected during the period $[1,t]$ as $D_{1,t}=\{(x_j, y_j)|j\in[1,t]\}$. If the data is collected in chunks, then the stream includes a set of chunks $D_{1,t}=\{\bar{D}_j|j\in[1,t]\}$, where each chunk $\bar{D}_j=\{(x_{jk}, y_{jk})|k\in[1, M]\}$ includes $M$ examples. Here, $x_{jk}$ represents an instance with $d$ dimensional attributes,  $y_{jk}$ denotes the corresponding label, and $M$ denotes the chunk size.
In this paper, we focus only on the data collected in chunks. 
If the stream $D_{1,t}$ follows a distribution $P_{1,t}(x,y)$, following \cite{survey}, we claim that a drift occurs at time
$t+1$ if 
\begin{equation}
P_{1,t}(x,y)\neq P_{t,\infty}(x,y).
\end{equation}

The goal of concept drift detection is to raise an alarm at time $t+1$ when the distribution of data changes.
The most popular metric for detecting the distribution change is prediction error.
For a classifier $f$, assuming the number of classes is $n$, the classifier will output a prediction probability for each class $\hat{y}\in \mathbb{R}^n$ and $\sum_{i=1}^n \hat{y}=1$. The prediction error of an instance $x_i$ is defined as:
\begin{equation}
    e_i = \mathbb{I}(\hat{y}_j = \arg\max_j f_j(x_i) \neq y_i),
\end{equation} where $\mathbb{I}(\cdot)$ is the indicator function. 

As we mentioned earlier, our motivation is that the prediction probability will intuitively change before the prediction error when drift occurs. In this paper, we measure the prediction probability by the PU-index which is defined as:

\begin{equation}
    u_i = 1-f_{y_i}(x_i),
\end{equation} where $f_{y_i}(x_i)$ denotes the probability predicted by the classifier that $x_i$ belongs to the ground truth class $y_i$.

\subsection{Theoretical Analysis of Error Rate and PU-index} \label{sec-prove}

To rigorously evaluate the efficacy of these two metrics for concept drift detection, we conduct a theoretical comparison from two complementary perspectives.
(1) When the PU-index distribution remains stable, potentially failing to detect concept drift, we investigate whether the error rate distribution exhibits changes that could indicate drift.
(2) Conversely, when the error rate distribution remains constant, we examine whether the PU-index distribution demonstrates changes that might reveal underlying changes in the data stream.

\begin{theorem}\label{theorem1}
    Let $W_1$ and $W_2$ be two windows of a data stream in a multi-class classification problem. If their respective PU-index histograms $H_1$ and $H_2$ are identical, where the histograms are constructed such that the first bin contains all misclassified instances and the remaining bins partition the misclassified instances, then the error rates and error standard deviations of $W_1$ and $W_2$ are equal.
\end{theorem}

\begin{theorem}\label{theorem2}
    Given a multi-class classification problem, if two windows have equal error standard deviations or error rates, their PU-index histograms, where the first bin contains all correctly classified instances and the remaining bins partition the misclassified instances, may not have identical bin proportions.
\end{theorem}
Due to the page limit, the proofs are provided in the Appendix.
These theorems lead to the following conclusions:
(1) Theorem \ref{theorem1} demonstrates that when the PU-index distribution remains stable, the error rate and the error standard deviation also remain constant. This implies that if the PU-index fails to detect concept drift, error-based metrics will also fail to detect it.
(2) Theorem \ref{theorem2} establishes that even when error rates and error standard deviations are equal between two windows, the PU-index distributions may differ. This suggests that the PU-index has the potential to detect subtle changes in the data distribution that are not captured by traditional error-based metrics.
These findings show that the PU-index offers at least the same sensitivity as error-based metrics and potentially higher sensitivity in certain scenarios, making it a more robust and comprehensive measure for detecting concept drift in streaming data environments.

\subsection{Sliding Window Strategy and Adaptive PU-index Bucketing Algorithm}

To detect concept drift using the PU-index without making it overly sensitive, we apply the Chi-square test to examine the PU-index distribution. The hypotheses are:

\textbf{Null Hypothesis ($\mathcal{H}_0$)}: The PU-index distribution does not change over time, indicating no concept drift.

\textbf{Alternative Hypothesis ($\mathcal{H}_1$)}: The PU-index distribution changes over time, indicating concept drift.

If the Chi-square statistic exceeds the critical value at the chosen significance level, we reject $\mathcal{H}_0$ and conclude that concept drift has occurred. Otherwise, we detect no significant drift. Thus, detecting concept drift using Chi-square involves two key steps: (1) partitioning the data stream into two windows and (2) constructing a histogram of the collected PU-index values.

We adopt a sliding window strategy to handle online streaming data. Let $D_{1,t}=\{\bar{D}_j|j\in[1,t]\}$ denote PU-index chunks collected from the start of the stream, potentially from different distributions. For instance, suppose $D_{1,t_1}$ and $D_{t_1,t}$ differ in distribution. If a new chunk $\bar{D}_{t+1}$ matches the distribution of $D_{t_1,t}$, no drift should be detected. However, keeping antiquated data $D_{1,t_1}$ could trigger a false alarm, since $D_{1,t_1}$ and $D_{t_1,t+1}$ differ in distribution.

To solve this "antiquated distribution" problem, we discard outdated data after detecting a drift at time $t_1$. Subsequent drift detection uses only $D_{t_1,t+1}$, avoiding false alarms caused by old data.
After discarding antiquated data, we must determine how to form two windows on the current substream. We do this by exploring all possible cutting points $r \in [t_1,t+1]$. Thus, $D_{t_1,t+1}$ is split into $D_{t_1,r}$ and $D_{r,t+1}$ for the Adaptive PU-index Bucketing algorithm. The sliding window is illustrated on the left side of Fig. \ref{fig-framework}.

Theoretical analysis shows that misclassified instances' PU-indices must be grouped into the same bin. For counterpart, we use Ei-kMeans to form bins that meet the Chi-square test requirements.  We call this the Adaptive PU-index Bucketing algorithm, illustrated in the middle of Fig. \ref{fig-framework}.

\subsection{PU-index based Drift Detector}

In this subsection, we introduce the overall of our method. Firstly, given a substream containing the recent chunks' PU-index $u_{t_1,t}$, we explore all cutting points $r\in [t_1,t]$. Based on the cutting points, we have $t-t_1$ window pairs, denoted as $\boldsymbol{u}_{t_1,r}$, and $\boldsymbol{u}_{r,t}$.
Then we defined the PU-index pairs for correctly and wrongly classified instances as:
\begin{align}
    \boldsymbol{u}^C_{t_1,r}&=\{u|u_i\in \boldsymbol{u}_{t_1,r}\land \hat{y}_i=y_i\}, \\
    \boldsymbol{u}^C_{r,t}&=\{u|u_i\in \boldsymbol{u}_{r,t}\land \hat{y}_i=y_i\}, \\
    \boldsymbol{u}^M_{t_1,r}&=\{u|u_i\in \boldsymbol{u}_{t_1,r}\land \hat{y}_i\neq y_i\}, \\
    \boldsymbol{u}^M_{r,t}&=\{u|u_i\in \boldsymbol{u}_{r,t}\land \hat{y}_i\neq y_i\}.
\end{align}
These four equations denote the PU-index of correctly classified instances on the first and second window, and the PU-index of misclassified instances on the first and second window respectively. 
Next, we compute the contingency table $T\in\mathbb{R}^{2\times (K+1)}$, where $K$ is a hyperparameter in Ei-kMeans.
To calculate $T$, firstly, we apply the Adaptive PU-index Bucketing algorithm on $\boldsymbol{u}^C_{t_1,r}$ to build a histogram. Then we count the instances from $\boldsymbol{u}_{t_1,r}^C$ falling into the bins of the histogram and fill them into $T_{1i}$ where $i\in[1, K]$. Likewise, we count the examples of $\boldsymbol{u}^C_{r,t}$ falling into the previously obtained bins and fill them into $T_{2i}$. Finally, we fill the $T_{1,K+1}$, and $T_{2,K+1}$ by the size of $\boldsymbol{u}^M_{t_1,r}$ and $\boldsymbol{u}^M_{r,t}$.
Then we define the expected frequency of $T_ij$ as:
\begin{equation}
E_{ij}=\frac{\sum_{j=1}^{K+1}T_{ij}\times\sum_{i=1}^{2}T_{ij}}{\sum_{ij}T_{ij}}. \label{eq-e}
\end{equation}

The Chi-square test statistic is defined as:
\begin{equation}
    \chi^2 = \sum_{i} \sum_j \left( \frac{T_{ij}^2}{E_{ij}} \right) - \sum_{ij}T_{ij}.
    \label{eq-test-statistic}
\end{equation}

\begin{figure}[tb]
  \centering
  \begin{subfigure}{.9\linewidth}
    \centering
    \includegraphics[width=\linewidth]{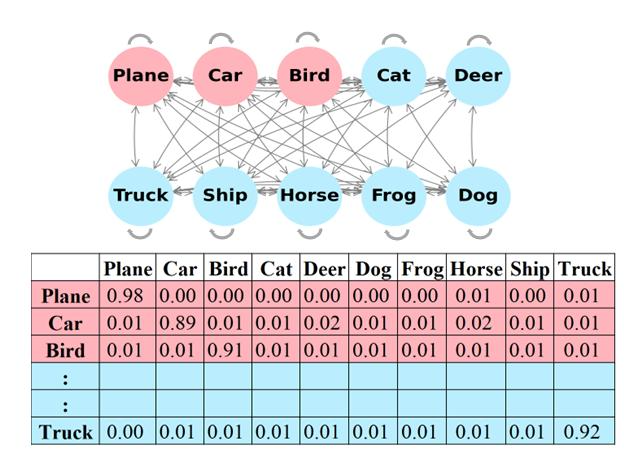}
  \end{subfigure}%
\caption{Illustrative example of generating CIFAR-10-CD through the Markov process. The classes marked in red represent the user's initial interest and are considered positive labels.  All other classes are considered negative labels. } 
\label{fig-cifar10-generate}
\end{figure}

\begin{table*}[tb]
\centering
\setlength{\tabcolsep}{1mm}
\begin{tabular}{c|l|cccccc|cccccc}
\toprule
\multicolumn{8}{c|}{Incremental Training} & \multicolumn{6}{c}{Training only at Initialization or Adaptation} \\
\toprule
Classifier & ddm name & airline-I & elec2-I & mixed-I & ps-I & sea0-I & sine-I & airline-O & elec2-O & mixed-O & ps-O & sea0-O & sine-O \\
\hline
\multirow{10}{*}{DNN} 
& ADWIN & 61.65 & 71.94 & 78.45 & 71.12 & 97.70 & 79.27 & 59.36 & 69.71 & 84.41 & 65.95 & 95.18 & 86.59 \\
& DDM & 61.29 & 71.02 & \textbf{80.72} & 71.37 & 96.86 & \underline{\textbf{86.41}} & 56.79 & 69.51 & 83.03 & 69.35 & 93.55 & 80.32 \\
& EDDM & 62.13 & 70.39 & 78.01 & 65.21 & 96.89 & 75.34 & \textbf{61.17} & 69.22 & 71.16 & 67.57 & 92.18 & 76.54 \\
& HDDM-A & 62.70 & 71.16 & 76.70 & 68.97 & 97.84 & 81.65 & 59.13 & 68.98 & 84.27 & 67.81 & 95.11 & \underline{\textbf{86.82}} \\
& HDDM-W & 61.50 & 71.23 & 77.07 & 68.74 & 97.84 & \textbf{85.43} & \underline{\textbf{62.42}} & 68.48 & 84.32 & 66.89 & 92.24 & 86.57 \\
& KSWIN & 63.02 & 70.56 & 78.58 & 70.80 & 97.88 & 79.00 & \textbf{61.34} & 69.08 & \textbf{84.43} & 66.46 & 91.60 & 83.66 \\
& PH & 62.15 & 72.25 & 79.49 & 70.86 & 97.73 & 78.07 & 60.36 & 68.24 & 84.36 & 68.83 & \textbf{95.21} & \textbf{86.88} \\
\cline{2-14}
\rule{0pt}{8pt}
& PUDD-1 & \textbf{63.31} & \textbf{74.92} & 77.39 & \underline{\textbf{72.25}} & \textbf{97.94} & \textbf{86.19} & 60.90 & 69.35 & 82.65 & \underline{\textbf{71.47}} & 94.89 & 83.39 \\
& PUDD-3 & \textbf{63.21} & \underline{\textbf{74.93}} & \textbf{80.05} & \textbf{72.23} & \textbf{98.04} & 85.12 & 60.16 & 68.98 & \textbf{84.65} & \textbf{70.37} & \textbf{95.99} & 84.97 \\
& PUDD-5 & \underline{\textbf{63.35}} & \textbf{74.92} & \underline{\textbf{82.81}} & \textbf{72.24} & \underline{\textbf{98.23}} & 82.51 & 60.19 & 68.68 & \underline{\textbf{84.90}} & \textbf{70.20} & \underline{\textbf{96.29}} & 85.09 \\
\hline
\rule{0pt}{8pt}
\multirow{10}{*}{GNB} 
& ADWIN & 50.17 & 68.90 & \textbf{83.95} & 70.06 & 94.18 & 82.49 & \textbf{54.66} & \textbf{68.30} & 83.62 & 68.78 & 93.93 & 82.45 \\
& DDM & 52.94 & 67.75 & 83.82 & 69.63 & 93.97 & 82.07 & 52.43 & 67.60 & \textbf{83.59} & 67.52 & 93.48 & 81.90 \\
& EDDM & \underline{\textbf{62.72}} & 67.73 & 83.19 & 70.04 & 94.06 & 83.12 & 54.11 & 67.64 & 74.65 & \textbf{70.04} & \textbf{94.06} & 73.88 \\
& HDDM-A & 52.80 & 67.73 & 83.92 & 70.87 & 94.28 & 83.25 & \textbf{\underline{55.62}} & \textbf{67.73} & 83.66 & \underline{\textbf{71.24}} & 93.96 & \textbf{83.36} \\
& HDDM-W & 48.66 & 67.73 & 83.91 & 71.06 & 92.11 & 82.83 & 48.62 & 67.71 & 83.60 & 69.53 & 91.63 & 83.25 \\
& KSWIN & 49.84 & 67.87 & 83.92 & 71.23 & 91.72 & 81.88 & 48.83 & 67.63 & 83.59 & 67.99 & 90.02 & 81.32 \\
& PH & 49.35 & 70.12 & 83.88 & 70.36 & 94.34 & \underline{\textbf{83.54}} & 49.02 & \underline{\textbf{70.04}} & 83.61 & 68.67 & 94.12 & 83.10 \\
\cline{2-14}
\rule{0pt}{8pt}
& PUDD-1 & \textbf{53.57} & \textbf{70.85} & 82.99 & \underline{\textbf{71.88}} & \textbf{94.61} & 83.12 & 51.05 & 62.76 & 79.23 & \textbf{71.13} & \textbf{94.25} & 81.48 \\
& PUDD-3 & \textbf{53.03} & \underline{\textbf{70.85}} & \textbf{83.92} & \textbf{71.59} & \textbf{94.81} & \textbf{83.39} & 49.45 & 59.32 & 83.58 & \textbf{71.20} & \textbf{94.60} & \textbf{83.80} \\
& PUDD-5 & 52.16 & \textbf{70.69} & \underline{\textbf{84.12}} & \textbf{71.59} & \underline{\textbf{94.85}} & \textbf{83.43} & \textbf{54.37} & 59.44 & \underline{\textbf{83.96}} & 70.40 & \underline{\textbf{94.62}} & \textbf{83.38} \\
\hline
\rule{0pt}{8pt}
\multirow{10}{*}{VFDT} 
& ADWIN & 60.39 & \textbf{73.98} & 84.30 & 71.69 & 94.77 & 87.11 & 61.22 & \textbf{74.29} & 83.54 & 68.78 & 92.96 & \textbf{85.74} \\
& DDM & 60.16 & \underline{\textbf{74.82}} & 84.14 & 70.68 & 94.86 & 86.50 & 59.28 & \underline{\textbf{74.75}} & 82.31 & 67.53 & 93.63 & 82.23 \\
& EDDM & 61.19 & 73.81 & 83.15 & 70.06 & 94.17 & 85.59 & \underline{\textbf{62.31}} & 73.81 & 73.73 & 70.06 & 93.60 & 77.72 \\
& HDDM-A & 60.95 & 73.90 & \textbf{84.40} & 70.84 & \textbf{95.20} & \textbf{87.53} & 60.29 & 73.83 & \textbf{83.66} & \underline{\textbf{71.24}} & 93.93 & 85.21 \\
& HDDM-W & 61.11 & 73.80 & \textbf{84.40} & 70.99 & 93.50 & \textbf{87.45} & \textbf{61.92} & 73.73 & 83.59 & 69.54 & 91.76 & 85.28 \\
& KSWIN & \textbf{61.30} & \textbf{74.10} & \underline{\textbf{84.42}} & 71.27 & 93.40 & 86.22 & \textbf{62.06} & \textbf{74.10} & 83.57 & 67.51 & 89.84 & 82.60 \\
& PH & 60.95 & 73.70 & 83.56 & 70.88 & 94.69 & 87.15 & 60.97 & 73.99 & 83.30 & 68.67 & 93.85 & 85.19 \\
\cline{2-14}
\rule{0pt}{8pt}
& PUDD-1 & 61.38 & 73.86 & 84.25 & \textbf{71.77} & 95.13 & 87.33 & 61.16 & 69.79 & 82.13 & \textbf{71.13} & \textbf{94.10} & 82.21 \\
& PUDD-3 & \underline{\textbf{61.57}} & 73.84 & 83.94 & \underline{\textbf{71.79}} & \textbf{95.21} & 87.42 & 59.90 & 69.79 & 84.04 & \textbf{71.20} & \underline{\textbf{94.63}} & \textbf{85.81} \\
& PUDD-5 & \underline{\textbf{61.57}} & 73.64 & 84.01 & \underline{\textbf{71.79}} & \underline{\textbf{95.24}} & \underline{\textbf{87.63}} & 57.04 & 71.83 & \underline{\textbf{84.16}} & 70.40 & \textbf{94.56} & \underline{\textbf{86.01}} \\
\hline
\end{tabular}
\caption{Comparative analysis against classic drift detectors across 3 synthetic and 3 real-world datasets. The top 3 results are highlighted in bold and the top 1 results are in both bold and underlined. PUDD-$x$ represents the threshold set as $10^{-x}$ for our method. The ps is short of powersupply dataset. Results for dataset sea10 and sea20 is provided in Appendix.}
\label{baseline}
\end{table*}

\begin{table}[t]
\setlength{\tabcolsep}{1mm}
\begin{tabular}{l|llllll}
\hline
\textbf{ddm name} & \textbf{airline}& \textbf{elec2}& \textbf{mixed}& \textbf{ps} & \textbf{sea0} & \textbf{sine} \\ \hline
AMF & 38.56 & 66.24 & 49.49 & 69.63 & 93.67 & 49.52 \\
IWE & 38.02 & 68.90 & 49.47 & 64.10 & 93.14 & 49.51 \\
NS& \textbf{67.91}& 76.42 & 81.09 & 72.39 & 93.54 & 91.01 \\
ADLTER& \textbf{\underline{70.00}} & 76.10 & 87.63 & 72.48 & 93.40 & \textbf{92.18}\\
MCD-DD& 63.65 & 69.81 & 86.68 & 71.66 & 97.66 & 90.21 \\ \hline
PUDD-1& 63.78 & \textbf{\underline{77.28}} & \textbf{\underline{89.51}} & \textbf{72.68}& \textbf{98.47}& \textbf{94.52}\\
PUDD-3& \textbf{64.62}& \textbf{76.77}& \textbf{89.47}& \textbf{\underline{72.79}} & \textbf{98.44}& \textbf{\underline{94.76}} \\
PUDD-5& 64.45 & \textbf{76.92}& \textbf{89.37}& \textbf{72.74}& \textbf{\underline{98.49}} & 90.90 \\ \hline
\end{tabular}
\caption{Comparison with SOTA methods. The dataset ps is short for powersupply. The results for dataset sea10 and sea20 is provided in Appnedix. The table shows that our methods PUDD in achieved top-1 in 5 out of 6 datasets, implying the effectiveness of PUDD compared with SOTA methods.}
\label{sota}
\end{table}

Finally, the p-value is computed by:
\begin{equation} \label{eq-chi}
    p = 1 - \int_0^{\chi^2} \frac{x^{\frac{K}{2}-1} \cdot e^{-\frac{x}{2}}}{2^{\frac{K}{2}} \cdot \Gamma\left(\frac{K}{2}\right)} \, dx, 
\end{equation}
Based on the Equation (\ref{eq-e}-\ref{eq-chi}), we can compute the p-value for each window pair $D_{t_1,r}$, and $D_{r,t}$.
If the minimum p-value among all window pairs is smaller than a predefined threshold $\sigma$, we raise a drift detected alarm. It is important to clarify that the specified threshold controls the Type I error rate for each individual window pair test rather than the overall Type I error across the entire substream. Consequently, we do not employ multiple comparison adjustments since our statistical guarantees apply at the single-test level rather than the family-wise level.
The pseudo-code and time complexity analysis is provided in the Appendix.

\section{Experiments}\label{sec-exp}
In this section, we introduce the settings and results of the experiments in our paper.
The details of implementation, datasets, baselines, and the critical difference diagrams \cite{cd} for the experiments in this paper are introduced in the Appendix.

\subsection{Datasets and Baselines}\label{sec-datasets}

We propose CIFAR-10-CD, a synthetic concept drift image dataset with an illustrative example in Fig. \ref{fig-cifar10-generate}, to simulate user interests changing via a Markov process. Initially, three CIFAR-10 classes are marked positive, with interest shifts occurring probabilistically (e.g. 1\% chance of Plane to Horse transfer). Our experiments utilize 3 real-world datasets (airline\cite{airline}, elec2\cite{elec}, powersupply\cite{powersupply}) and 4 synthetic sets (sine\cite{ddm}, mixed\cite{ddm}, CIFAR-10-CD, sea variants\cite{moa}). We compare against 7 classic detectors (ADWIN\cite{adwin}, DDM\cite{ddm}, EDDM\cite{EDDM}, HDDM-A\cite{HDDM}, HDDM-W\cite{HDDM}, KSWIN\cite{kswin}, PH\cite{ph}) and 5 SOTA methods (MCDD\cite{mcdd}, AMF\cite{amf}, IWE\cite{iwe}, NS\cite{ns},
and ADLTER\cite{adlter}).

\subsection{Comparison with Baselines and Ablation Studies}\label{sec-exp-main}

In this subsection, we compare our method with 7 classic drift detectors and 5 SOTA methods on 9 datasets (including a variant of the SEA dataset). Due to page constraints, results for SEA10 and SEA20 appear in the Appendix. We evaluate all methods using three classifiers—DNN (architecture detailed in the Appendix), Gaussian Naive Bayes (GNB) \cite{scipy}, and VFDT \cite{vfdt}—under two training regimes: incremental (dataset-I) and one-time training at an alarm (dataset-O).

Results for the comparison with classic detectors are presented in Table \ref{baseline}, and those for SOTA methods are given in Table \ref{sota}. For CIFAR-10-CD, due to its learning complexity, we only use incremental training and report results in Fig. \ref{fig:CIFAR-10}. Our method is denoted as \textbf{PUDD-X}, where X represents the exponent in $10^{-X}$.
Based on these experiments, we derive 6 observations. We introduce 4 of them here and leave the remaining 2 in the Appendix.


\textbf{Observation 1: our method shows stronger performance compared to classic drift detectors} as evidenced by the results presented in Table \ref{baseline}, Fig. \ref{fig:CIFAR-10}, and additional results in Appenidx.  In incremental learning settings, PUDD ranks first in 17 out of 24 cases across different datasets and classifiers, and it is in the top 3 in 20 out of 24 cases. When trained only initially or with adaptation, it still performs well. It achieves first rank in 15 cases and top 3 in 19 cases. This shows that PUDD is particularly effective with incremental training. Results in Fig. \ref{fig:CIFAR-10} show PUDD outperformed all the baselines, which demonstrates the superiority of our method in detecting the concept drift in the image dataset. 

\textbf{Observation 2: PUDD performs better with a smaller threshold} as revealed in Table \ref{baseline} and additional results in Appendix. In incremental learning scenarios, PUDD-1, PUDD-3, and PUDD-5 achieve top 1 in 5, 5, and 8 cases respectively. When tested in training only once until alarm way, PUDD-1, PUDD-3, and PUDD-5 achieved top 1 in 2, 6, and 8 cases respectively. PUDD consistently shows improved performance at lower thresholds in both scenarios. As detailed in the Sensitivity of PU-index section, a drift alarm triggers when the p-value is below the threshold, with lower thresholds indicating stricter conditions for alarm detection. Therefore, PUDD's better performance with smaller thresholds suggests a high sensitivity to drift.

\textbf{Observation 3: PUDD shows very competitive performance compared to SOTA methods.}
As shown in Table \ref{sota} and additional results in Appendix, our method attains the top rank in 7 out of 8 cases. On certain datasets, this improvement is particularly pronounced. For instance, PUDD-5 achieves a 98.49\% accuracy, which is 2.8\% higher than the best SOTA method. The only exception occurs in the airline dataset, where NS and ADLTER outperform PUDD.

This discrepancy can be explained by the airline dataset’s tabular nature and its numerous attributes, which are more effectively modeled through tree-based ensemble learning utilized by these SOTA methods. Moreover, these methods adapt to drift by adjusting ensembles rather than discarding and retraining them. In contrast, PUDD relies on retraining classifiers solely on recent data, which may not be suitable for attribute-rich datasets like the airline dataset. Nevertheless, for all other datasets, the results confirm that PUDD surpasses SOTA methods, thereby underscoring its overall superiority.

\textbf{Observation 4: The Adaptive PU-index Bucketing algorithm outperforms Ei-kMeans.}  
Figure \ref{fig-ek} shows that PUDD  surpasses Ei-kMeans across various datasets, classifier training methods, and threshold settings. The critical difference diagram in the Appendix shows that improvements at thresholds $10^{-3}$ and $10^{-5}$ are statistically significant.

In summary, these results confirm the theoretical benefits of the PU-index for drift detection. PUDD outperforms both classic and SOTA detectors, and the Adaptive PU-index Bucketing algorithm shows significant improvements over Ei-kMeans. This validates the PU-index as a sensitive, robust indicator capable of detecting drift even when error rates remain unchanged, thereby overcoming a major shortcoming of error rate-based approaches.

\begin{figure}[tb]
\centering 
\includegraphics[width=\linewidth]{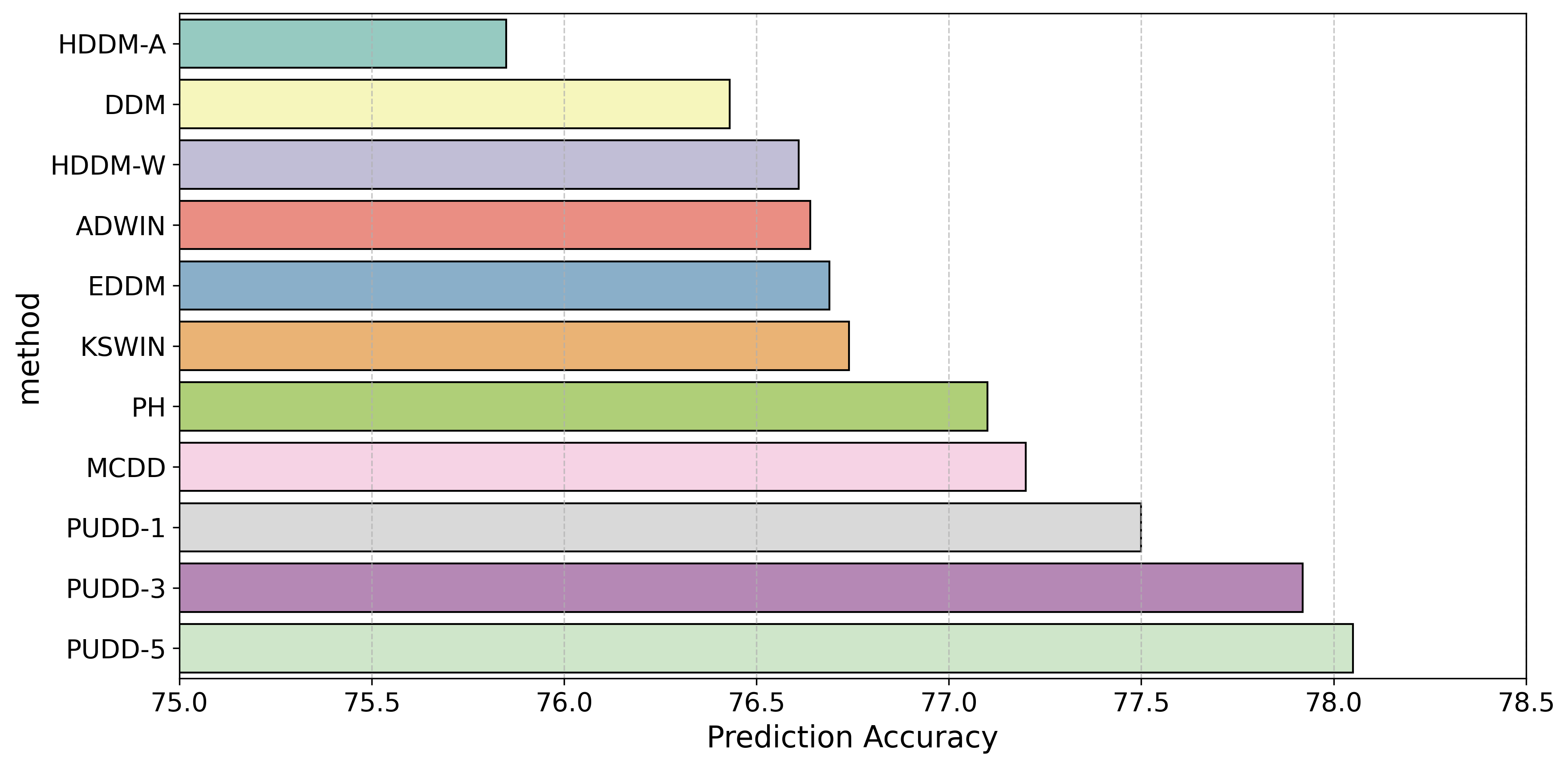}
\caption{Comparison with baselines on CIFAR-10-CD, excluding methods unable to detect drift in image datasets.} 
\label{fig:CIFAR-10} 
\end{figure}

\begin{figure}[tb]
  \centering
  \begin{subfigure}{\linewidth}
    \centering
    \includegraphics[width=\linewidth]{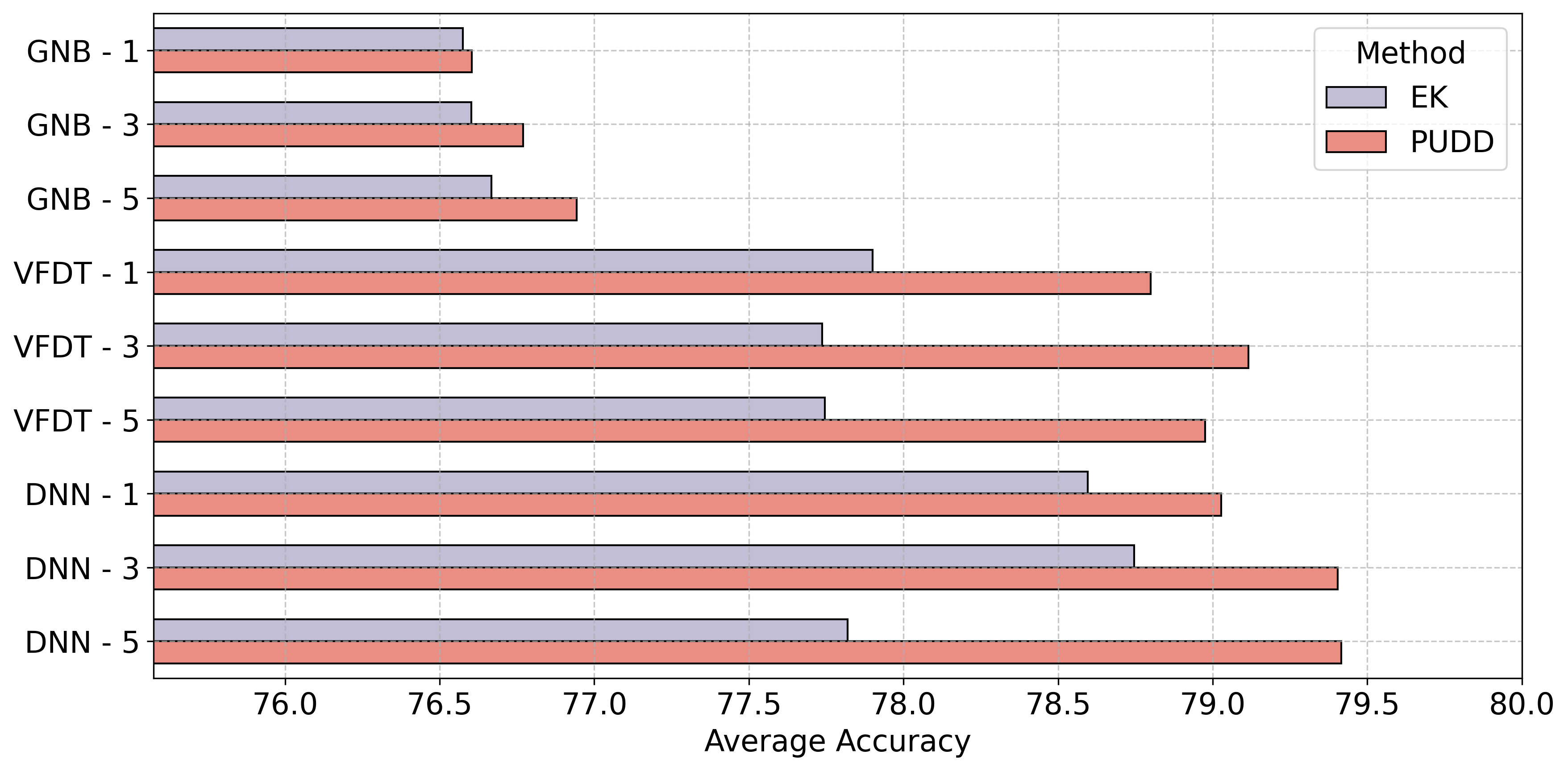}
  \end{subfigure}%
\caption{Accuracy comparison between PUDD (using Adaptive PU-index Bucketing) and Ei-kMeans (EK). We show average accuracy across 9 datasets using 3 classifiers.} 
\label{fig-ek}
\end{figure}

\section{Conclusion and Future Work}\label{sec:conclusion}
In our study, we demonstrated that the PU-index, as opposed to the error rate, is a more effective measure for detecting concept drift in machine learning models. We utilized the Adaptive PU-index bucketing algorithm to partition the PU-index and the Chi-square test to detect concept drift.
We also introduced a technique for inducing concept drift in image datasets by simulating changes in user interest. We validated our method through experiments on both synthetic and real-world datasets. 
Future work should focus on automating drift alarm threshold determination, as current methods rely on manual settings that may not remain optimal over time.  Our research also uncovers a method for generating multi-stream concept drift in image datasets by emulating shifts in user interests using Markov matrices, offering valuable insights for research in multistream concept drift learning.

\section*{Acknowledgments}\label{sec:ack}
This work was supported by the Australian Research Council through the Laureate Fellow Project under Grant FL190100149.

\bibliography{Extended-version}

\newpage

\appendix

\appendix

\twocolumn[{

\section{Proof of Identical Error Rate when Concept Drift Occurs}\label{proof0}

\begin{definition}
Let $\mathcal{X}$ be an input space and $\mathcal{Y}$ be an output space. A \textit{concept} is a function $c:\mathcal{X}\to\mathcal{Y}$ that assigns to each input $x\in \mathcal{X}$ a corresponding label $c(x)\in\mathcal{Y}$. A classifier $h:\mathcal{X}\to\mathcal{Y}$ is a learned hypothesis that attempts to approximate a given concept. 

We assume a data-generating distribution $D$ defined over $\mathcal{X}$, with $x\sim D$. At any time, the \textit{true concept} $c$ determines the label for each $x$. If concept drift occurs, the concept changes from an old concept $c_{old}$ to a new concept $c_{new}$, potentially altering the mapping from $\mathcal{X}$ to $\mathcal{Y}$. We assume the input distribution $D$ remains fixed, i.e., $p(x)$ is unchanged, but $c_{old}(x)$ and $c_{new}(x)$ may differ on some subset of $\mathcal{X}$.
\end{definition}

\begin{definition}
A loss function $\mathcal{L}:\mathcal{Y}\times \mathcal{Y}\to\mathbb{R}_{\ge 0}$ measures the discrepancy between the predicted label $h(x)$ and the true label $c(x)$. For example:
\begin{itemize}
    \item \textbf{0-1 loss:} 
    \[
    \mathcal{L}(h(x), c(x)) = 
    \begin{cases}
    0, & \text{if } h(x)=c(x) \\[6pt]
    1, & \text{otherwise}
    \end{cases}.
    \]

    \item \textbf{MSE loss:} 
    \[
    \mathcal{L}(h(x), c(x)) = (h(x)-c(x))^2.
    \]
\end{itemize}

The \textit{error rate} or \textit{generalization error} of $h$ under concept $c$ and distribution $D$ is given by:
\[
\mathbb{E}_{x\sim D}[\mathcal{L}(h(x), c(x))].
\]
\end{definition}

\begin{theorem}
Concept drift can occur even if the error rate under a given loss function remains unchanged. More formally, there exist scenarios such that $c_{old}(x)\neq c_{new}(x)$ for some $x\in \mathcal{X}$, yet:
\[
\mathbb{E}_{x\sim D}[\mathcal{L}(h(x), c_{new}(x))] = \mathbb{E}_{x\sim D}[\mathcal{L}(h(x), c_{old}(x))]
\]
holds for both 0-1 and MSE losses.
\end{theorem}

\begin{proof}
Consider a classifier $h:\mathcal{X}\to\mathcal{Y}$ and an unchanged distribution $D$. Let $c_{old}$ and $c_{new}$ be the old and new concepts, respectively, and define:

$$X_{same}=\{x\mid c_{old}(x)=c_{new}(x)\}, $$

$$X_{diff}=\{x\mid c_{old}(x)\neq c_{new}(x)\}.$$

\noindent\textbf{Case 1: 0-1 Loss.} We have:

\begin{align*}
\mathbb{E}_{x\sim D}[\mathcal{L}(h(x), c_{new}(x))] 
&= \mathbb{E}_{x\sim D}[\mathcal{L}(h(x), c_{new}(x)) \cdot \mathbf{1}_{X_{same}}]
+ \mathbb{E}_{x\sim D}[\mathcal{L}(h(x), c_{new}(x)) \cdot \mathbf{1}_{X_{diff}}]\\[6pt]
&= \mathbb{E}_{x\sim D}[\mathcal{L}(h(x), c_{old}(x)) \cdot \mathbf{1}_{X_{same}}]
+ \mathbb{E}_{x\sim D}[\mathcal{L}(h(x), c_{new}(x)) \cdot \mathbf{1}_{X_{diff}}]\\[6pt]
&= \mathbb{E}_{x\sim D}[\mathcal{L}(h(x), c_{old}(x))]
+ \mathbb{E}_{x\sim D}[(\mathcal{L}(h(x), c_{new}(x)) - \mathcal{L}(h(x), c_{old}(x))) \cdot \mathbf{1}_{X_{diff}}].
\end{align*}

If $\mathbb{E}_{x\sim D}[(\mathcal{L}(h(x), c_{new}(x))-\mathcal{L}(h(x), c_{old}(x)))\cdot\mathbf{1}_{X_{diff}}]=0$, then:
\[
\mathbb{E}_{x\sim D}[\mathcal{L}(h(x), c_{new}(x))] 
= \mathbb{E}_{x\sim D}[\mathcal{L}(h(x), c_{old}(x))],
\]
even though $X_{diff}\neq\emptyset$. Thus, concept drift can occur with no change in error rate under 0-1 loss.

\noindent\textbf{Case 2: MSE Loss.} Let $\Delta(x)=c_{new}(x)-c_{old}(x)$. Then:
\begin{align*}
\mathbb{E}_{x\sim D}[(h(x)-c_{new}(x))^2]
&= \mathbb{E}_{x\sim D}[(h(x)-c_{old}(x)-\Delta(x))^2]\\[6pt]
&= \mathbb{E}_{x\sim D}[(h(x)-c_{old}(x))^2]
+ \mathbb{E}_{x\sim D}[\Delta(x)^2 + 2\Delta(x)(c_{old}(x)-h(x))].
\end{align*}

If $\mathbb{E}_{x\sim D}[\Delta(x)^2 + 2\Delta(x)(c_{old}(x)-h(x))]=0$, then:
\[
\mathbb{E}_{x\sim D}[(h(x)-c_{new}(x))^2]
= \mathbb{E}_{x\sim D}[(h(x)-c_{old}(x))^2],
\]
despite the existence of $\Delta(x)\neq 0$ for some $x$. Hence, under MSE loss, concept drift can also remain undetected by the error rate.

In both cases, we find scenarios where concept drift does not manifest as a change in expected error, completing the proof.

\end{proof}

}]

\section{Table of notation}
\begin{table}[h]
\centering
\begin{tabular}{|c|p{0.6\linewidth}|}
\hline
\textbf{Symbol} & \textbf{Description} \\
\hline
$D_{1,t}$ & Streaming data collected during the period $[1, t]$ \\
$(x_j, y_j)$ & Instance $j$ with features $x_j$ and label $y_j$ \\
$\bar{D}j$ & Chunk $j$ of data \\
$M$ & Chunk size \\
$P_{1,t}(x, y)$ & Distribution of data in the period $[1, t]$ \\
$f$ & Classifier \\
$\hat{y}$ & Predicted class probabilities \\
$n$ & Number of classes \\
$e_i$ & Prediction error of instance $x_i$ \\
$\mathbb{I}(\cdot)$ & Indicator function \\
$u_i$ & Uncertainty of instance $x_i$ \\
$f_{y_i}(x_i)$ & Predicted probability of instance $x_i$ belonging to its true class $y_i$ \\
$O_{ij}$ & Observed frequency in the $i$-th row and $j$-th column of the contingency table \\
$E_{ij}$ & Expected frequency for the cell at the $i$-th row and $j$-th column  \\
$H_i$ & Histogram built for window $i$ \\
$N$ & Total observed frequencies \\
$\chi^2$ & Chi-square test statistic \\
$p$ & p-value associated with $\chi^2$ \\
$w$ & Degrees of freedom \\
$K$ & Hyperparameter in Ei-kMeans \\
$\sigma$ & Predefined threshold for the Chi-square test \\ 
\hline
\end{tabular}
\caption{Notation table}
\label{tab:notation}
\end{table}

\section{Proof of Theorem 1}\label{proof1}
\begin{proof}
Let $W$ be a window containing a subset of instances from the data stream $D_{1,t}$.
Let $\mathcal{Y} = \{1, 2, ..., n\}$ be the set of class labels, where $n$ is the number of classes.
For an instance $x_i$ in $W$, let $\hat{y}_i = \arg\max_{j \in \mathcal{Y}} f_j(x_i)$ be the predicted class.
Define the partition $\mathcal{B} = \{B_0, B_1, ..., B_m\}$ of $W$ as follows:
$$B_0 = \{(x_i, y_i) \in W | \hat{y}_i = y_i\}$$
$$\bigcup_{j=1}^m B_j = \{(x_i, y_i) \in W | \hat{y}_i \neq y_i\}$$
where $B_j$ are pairwise disjoint, i.e., $B_i \cap B_j = \emptyset$ for $i \neq j$. Moreover, $\bigcup_{j=1}^m B_j$ covers the entire range of PU-index values [0, 1] for misclassified instances

Let $H_W$ be the histogram of $W$ constructed based on this partition, where each bin corresponds to a set in $\mathcal{B}$.
Given $H_{W_1} = H_{W_2}$, it implies
$$\frac{|B_j^{(1)}|}{|W_1|} = \frac{|B_j^{(2)}|}{|W_2|}, \forall j \in \{0, 1, ..., m\},$$ where $|\cdot|$ denotes set cardinality.
The error rate for window $W$ is defined as:
$$\bar{e}_W = \frac{1}{|W|} \sum_{(x_i, y_i) \in W} e_i = \frac{1}{|W|} \sum_{j=1}^m |B_j| = 1 - \frac{|B_0|}{|W|}$$
Therefore,
$\bar{e}_{W_1} = 1 - \frac{|B_0^{(1)}|}{|W_1|} = 1 - \frac{|B_0^{(2)}|}{|W_2|} = \bar{e}_{W_2}$, which implies that the error rate of $W_1$ and $W_2$ are equal.

The error standard deviation for window $W$ is defined as:
\begin{equation}
\begin{aligned}
\sigma_W &= \sqrt{\frac{1}{|W|} \sum_{(x_i, y_i) \in W} (e_i - \bar{e}_W)^2} \\
         &= \sqrt{\frac{1}{|W|} \sum_{(x_i, y_i) \in W} (e_i^2 - 2e_i\cdot\bar{e}_W + \bar{e}_W^2)} \\
         &= \sqrt{\frac{1}{|W|} \sum_{(x_i, y_i) \in W} e_i^2 - 2\bar{e}_W \cdot \frac{1}{|W|} \sum_{(x_i, y_i) \in W} e_i + \bar{e}_W^2} \\
         &= \sqrt{\frac{1}{|W|} \sum_{(x_i, y_i) \in W} e_i - 2\bar{e}_W^2 + \bar{e}_W^2} \\
         &= \sqrt{\bar{e}_W - \bar{e}_W^2}
\end{aligned}
\end{equation}
Since $\bar{e}_{W_1} = \bar{e}_{W_2}$, we have $\sigma_{W_1} = \sigma_{W_2}$.

Therefore, the error rates and error standard deviations of $W_1$ and $W_2$ are equal.

\end{proof}

\section{Proof of Theorem 2}\label{proof2}
\begin{proof}
We prove the theorem by providing a counterexample in binary classification case. Consider two windows $W_1$ and $W_2$:

\begin{itemize}
\item In $W_1$, 50\% of instances have uncertainty in $(0.9, 1.0]$, and the remaining 50\% have uncertainty in $[0.0, 0.1]$.
\item In $W_2$, 50\% of instances have uncertainty in $(0.8, 0.9]$, and the remaining 50\% have uncertainty in $(0.1, 0.2]$.
\end{itemize}

Let $B_1 = \{b: b \subseteq [0, 0.5]\}$ and $B_2 = \{b: b \subseteq (0.5, 1]\}$ denote the sets of histogram bins below and above the uncertainty value of 0.5, respectively.

In both windows, 50\% of instances fall into bins in $B_2$ (uncertainty $>$ 0.5) and are misclassified, while the remaining 50\% fall into bins in $B_1$ (uncertainty $\leq$ 0.5) and are correctly classified. Thus, the error rates of both windows are equal:

\begin{align}
\bar{e}_1 &= \frac{1}{n^{(1)}} \sum_{b \in B_2} n_b^{(1)} = 0.5, \\
\bar{e}_2 &= \frac{1}{n^{(2)}} \sum_{b \in B_2} n_b^{(2)} = 0.5.
\end{align}

where $n^{(1)}$ and $n^{(2)}$ are the total numbers of instances in $W_1$ and $W_2$, respectively, and $n_b^{(1)}$ and $n_b^{(2)}$ denote the number of instances falling into bin $b$ in $W_1$ and $W_2$, respectively. In this case, the error rates of the two windows are equal.

Consequently, the error standard deviations of both windows are also equal:

\begin{align}
\text{std}_1 &= \sqrt{\frac{1}{n^{(1)}} \sum_{i=1}^{n^{(1)}} (e_i^{(1)} - \bar{e}_1)^2}\notag\\
& = \sqrt{\bar{e}_1 (1 - \bar{e}_1)} = 0.5, \\
\text{std}_2 &= \sqrt{\frac{1}{n^{(2)}} \sum_{i=1}^{n^{(2)}} (e_i^{(2)} - \bar{e}_2)^2} \notag\\
& = \sqrt{\bar{e}_2 (1 - \bar{e}_2)} = 0.5.
\end{align}

However, the bin proportions of the uncertainty histograms $H_1$ and $H_2$ are different:

\begin{itemize}
\item In $H_1$, 50\% of instances fall into the bin $(0.9, 1.0]$, and 50\% fall into the bin $[0.0, 0.1]$.
\item In $H_2$, 50\% of instances fall into the bin $(0.8, 0.9]$, and 50\% fall into the bin $(0.1, 0.2]$.
\end{itemize}

Therefore, even though the error standard deviations or error rates of $W_1$ and $W_2$ are equal, their uncertainty histograms have different bin proportions, which contradicts the statement that equal error standard deviations imply identical bin proportions in the uncertainty histograms.
\end{proof}

\section{Pseudo code and time complexity of PUDD}
\begin{algorithm}[h]
    \caption{Adaptive PU-index Bucketing algorithm drift detector based on PU-index.}
    \begin{algorithmic}[1]
    \STATE \textbf{Input}: Chi-square test threshold $\sigma$, PU-index of recent examples $\boldsymbol{u}_{t_1,t}$ derived from sliding window strategy, hyperparameter $K$.
    \FOR{$r=t_1, ..., t$} \label{a-for}
    \STATE Derive $\boldsymbol{u}^C_{t_1,r},\boldsymbol{u}^C_{r,t},\boldsymbol{u}^M_{t_1,r}, \boldsymbol{u}^M_{r,t}$ from equation (1-4).   \label{a-split}
    \IF{$\mathrm{Mean}(\boldsymbol{u}^M_{t_1,r}) \geq \mathrm{Mean}(\boldsymbol{u}^M_{r,t})$} \label{a-skip}
    \STATE Initialize $K$ centroids $C$ based on Ei-kMeans centroids initialization algorithm on $\boldsymbol{u}^C_{t_1,r}$. \label{a-main}
    \STATE Build clusters based on Ei-kMeans partitioning algorithm on $\boldsymbol{u}^C_{t_1,r}$ based on initial centroids $C$.
    \STATE For each cluster, count the frequencies of $\boldsymbol{u}^C_{t_1,r}$ and $\boldsymbol{u}^C_{t_1,r}$. Subsequently, calculate the sizes of $\boldsymbol{u}^M_{t_1,r}$ and $\boldsymbol{u}^M_{r,t}$. Utilize these values to populate the contingency table accordingly.
    \STATE Conduct a Chi-square test on the contingency table and derive p-value $p$ by Eq (14).
    \STATE $P=P\cup \{p\}$. \label{a-mainend}
    \ENDIF \label{a-endif}
    \ENDFOR
    \IF{$P \neq \emptyset $ and $\min P<\sigma$} \label{a-raise-condition}
    \STATE Raise drift detected alarm.
    \ENDIF
    \end{algorithmic}
    \label{alg:ours}
    \end{algorithm}

    The algorithm framework of the PU-index-based drift detector is provided in Algorithm \ref{alg:ours}. Assuming the current time step is $t$ and the last alarm is raised at time step $t_1$, thus there are $t-t_1$ chunks in the current sub-stream.
    The first step is to split these chunks into two windows.
    In line \ref{a-for}, we search all the cutting points from the index $t_1$ to $t$.
    Then we apply the Adaptive PU-index Bucketing-based drift detector on the derived two samples from line \ref{a-split}-\ref{a-endif}. Notice that this brutal-force search can be processed simultaneously by multiprocess processing in implementation. In line \ref{a-skip}, we compare the mean of PU-index on the two samples $\boldsymbol{u}^M_{t_1,r}$ and $\boldsymbol{u}^M_{r,t}$, and we skip the subsequent p-value computation if the mean of $\boldsymbol{u}^M_{t_1,r}$ is larger than that of $\boldsymbol{u}^M_{r,t}$ to speed up the algorithm.
    From line \ref{a-main}-\ref{a-mainend}, we built clusters and a contingency table to derive the p-value for each cutting point. At line \ref{a-raise-condition}, we investigate whether the minimum p-value is smaller than the predefined threshold. If so, the null hypothesis test is rejected and we claim the concept drift has occurred.
    The time complexity for the lines \ref{a-split} is $O(n)$, where $n$ is the number of examples in the sliding window. The time complexity for the lines \ref{a-main}-\ref{a-mainend} is $O(K^2 n\log n)$ as introduced in \cite{eikmeans}. Thus, the total time complexity of our algorithm is $O(K^2 n\log n)$.

\section{Pseudo code for online instance based PUDD}
Firstly, we introduce how to incrementally update the key variables, clusters, and contingency tables. Then we provide the pseudo code of online instance based PUDD algorithm.

\textbf{Incremental Updates of Sets.}
At any time $t$, for each candidate cut point $r \in [t_1, t]$, we have the following sets: 
$u^C_{t_1,r}(t)$, $u^M_{t_1,r}(t)$, $u^C_{r,t}(t)$, and $u^M_{r,t}(t)$, derived from the PU-index.

Here, $u^C_{t_1,r}(t)$ and $u^M_{t_1,r}(t)$ represent the correctly and misclassified instances, respectively, in the substream from $t_1$ to $r$, while $u^C_{r,t}(t)$ and $u^M_{r,t}(t)$ represent the correctly and misclassified instances, respectively, in the substream from $r+1$ to $t$.

When time advances to $t+1$, we receive a new instance $u_{t+1}$ with prediction $\hat{y}_{t+1}$ and true label $y_{t+1}$. We update the sets for all $r \in [t_1, t+1]$ as follows.\\

\begin{equation}\label{eq:update_uc_t1r}
u^C_{t_1,r}(t+1) =
\begin{cases}
u^C_{t_1,r}(t), & r < t+1  \\[6pt]
u^C_{t_1,r}(t) \cup \{u_{t+1}\}, & r = t+1 \text{ and } \hat{y}_{t+1}=y_{t+1}\\
u^C_{t_1,r}(t), & r = t+1 \text{ and } \hat{y}_{t+1}\neq y_{t+1}
\end{cases}
\end{equation}

\begin{equation}\label{eq:update_um_t1r}
u^M_{t_1,r}(t+1) =
\begin{cases}
u^M_{t_1,r}(t), & r < t+1 \\[6pt]
u^M_{t_1,r}(t), & r = t+1 \text{ and } \hat{y}_{t+1}=y_{t+1}\\[6pt]
u^M_{t_1,r}(t) \cup \{u_{t+1}\}, & r = t+1 \text{ and } \hat{y}_{t+1}\neq y_{t+1}
\end{cases}
\end{equation}

\begin{equation}\label{eq:update_uc_rt}
u^C_{r,t+1}(t+1) =
\begin{cases}
u^C_{r,t}(t) \cup \{u_{t+1}\}, & r < t+1 \text{ and } \hat{y}_{t+1}=y_{t+1}\\[6pt]
u^C_{r,t}(t), & r < t+1 \text{ and } \hat{y}_{t+1}\neq y_{t+1}\\[6pt]
\emptyset, & r = t+1
\end{cases}
\end{equation}

\begin{equation}\label{eq:update_um_rt}
u^M_{r,t+1}(t+1) =
\begin{cases}
u^M_{r,t}(t) \cup \{u_{t+1}\}, & r < t+1 \text{ and } \hat{y}_{t+1}\neq y_{t+1}\\[6pt]
u^M_{r,t}(t), & r < t+1 \text{ and } \hat{y}_{t+1}=y_{t+1}\\[6pt]
\emptyset, & r = t+1
\end{cases}
\end{equation}

In these definitions, if $r < t+1$, the first window $\boldsymbol{u}_{t_1,r}$ remains unchanged, and we only incrementally update the second window $\boldsymbol{u}_{r,t+1}$ by adding the new instance $u_{t+1}$ according to its classification correctness. If $r = t+1$, a new cut point is introduced, and the first window $\boldsymbol{u}_{t_1,r}$ is augmented by $u_{t+1}$ if it is correctly classified, while the second window $\boldsymbol{u}_{r,t+1}$ is initially empty.

\textbf{Incremental Updates of Clusters.}
For each $r \in [t_1, t]$, at time $t$, we have clusters:
\begin{equation}\label{eq:clusters_t}
C_{t_1,r}(t) = \{C_1(t), C_2(t), \dots, C_K(t)\},
\end{equation}
obtained by applying Ei-kMeans initialization and subsequent kMeans clustering on $u^C_{t_1,r}(t)$.\\
When moving to time $t+1$, the incremental update is:
\begin{equation}\label{eq:update_clusters}
C_{t_1,r}(t+1) =
\begin{cases}
C_{t_1,r}(t), & r < t+1\\[6pt]
\text{kMeans}(u^C_{t_1,r}(t+1), C^{\text{init}}), & r = t+1
\end{cases}
\end{equation}

If $r < t+1$, no new instances affect the first window $\boldsymbol{u}_{t_1,r}$, so we can directly use the previously cached clusters. If $r = t+1$, the arrival of $u_{t+1}$ modifies $u^C_{t_1,r}(t+1)$, requiring a fresh Ei-kMeans initialization and kMeans run. Once computed, $C_{t_1,r}(t+1)$ is cached for future reuse.

\textbf{Incremental Updates of Contingency Tables.}
For a given $r \in [t_1, t]$, at time $t$, we have a contingency table:
\begin{equation*}\label{eq:table_t}
T_{t_1,r}(t) \in \mathbb{R}^{2 \times (K+1)},
\end{equation*}
where the first row corresponds to $(u^C_{t_1,r}(t), u^M_{t_1,r}(t))$ and the second row corresponds to $(u^C_{r,t}(t), u^M_{r,t}(t))$. The first $K$ columns represent the counts of correctly classified instances in each of the $K$ clusters, and the last column stores the counts of misclassified instances.

When time advances to $t+1$, we update $T_{t_1,r}(t)$ as:
\begin{equation}\label{eq:update_table}
T_{t_1,r}(t+1) =
\begin{cases}
T_{t_1,r}(t) + \Delta T, & r < t+1 \\[6pt]
\text{Compute Contingency Table}, & r = t+1
\end{cases}
\end{equation}

The term $\Delta T$ accounts for the incremental contribution of the new instance $u_{t+1}$. To compute $\Delta T$, we first determine whether $u_{t+1}$ falls into the first or second window, depending on $r$. Next, we check if $u_{t+1}$ is correctly classified or not. If it is correctly classified, we identify its assigned cluster using $C_{t_1,r}(t+1)$, and increment the corresponding cluster count in the table. If it is misclassified, we increment the misclassified count in the last column. This process involves only a single increment operation per new instance, making it highly efficient.

If $r = t+1$, since the composition of $u^C_{t_1,r}(t+1)$ (and possibly $u^M_{t_1,r}(t+1)$) changes due to the arrival of $u_{t+1}$, we must fully recompute the contingency table according to $u^C_{t_1,r}(t+1)$, $u^M_{t_1,r}(t+1)$,  $u^C_{r,t+1}(t+1)$, $u^M_{r,t+1}(t+1)$. After recomputation, $T_{t_1,r}(t+1)$ is cached.

\begin{algorithm}[t]
\caption{Online Instance based PU-index Bucketing-based Drift Detector}
\label{alg:incremental_pudd}
\begin{algorithmic}[1]
\REQUIRE Current $t+1$, last alarm $t_1$, threshold $\sigma$, $K$, $\boldsymbol{u}_{t_1,t}$, cached results.
\STATE $P \gets \emptyset$

\FOR{$r = t_1 \to t+1$}
    \STATE Update $u^C_{t_1,r}(t+1), u^M_{t_1,r}(t+1), u^C_{r,t+1}(t+1),$ \\$u^M_{r,t+1}(t+1)$ using Eqs.\ (\ref{eq:update_uc_t1r})--(\ref{eq:update_um_rt}) \label{algo2-update}
    \STATE Compute $C_{t_1,r}(t+1)$ using Eq.\ (\ref{eq:update_clusters})
    \STATE Compute $T_{t_1,r}(t+1)$ using Eq.\ (\ref{eq:update_table})
    \STATE Compute Chi-square p-value $p_r$
    \STATE $P = P \cup \{p_r\}$
\ENDFOR

\IF{$P \neq \emptyset$ \textbf{ and } $\min(P)<\sigma$}
    \STATE Raise drift alarm
\ENDIF
\end{algorithmic}
\end{algorithm}

\textbf{Time complexity analysis.} By leveraging caching and incremental updates as defined in Eqs.\ (\ref{eq:update_uc_t1r})--(\ref{eq:update_um_rt}), (\ref{eq:update_clusters}), and (\ref{eq:update_table}), our online instance-based PUDD algorithm achieves efficient updates with minimal computational overhead. When previously computed results are available, there is no need for a full recomputation of clusters or contingency tables, thus enabling near-constant-time updates at most time steps.

The pseudo code of the online instance-based PUDD algorithm is shown in Algorithm~\ref{alg:incremental_pudd}. Consider the computational cost at each relevant line (as annotated in the algorithm):

\begin{itemize}
\item Updating $u^C_{t_1,r}(t+1), u^M_{t_1,r}(t+1), u^C_{r,t+1}(t+1), u^M_{r,t+1}(t+1)$ using Eqs.\ (\ref{eq:update_uc_t1r})--(\ref{eq:update_um_rt}) requires $O(1)$ time (line~\ref{algo2-update}).
\item Recomputing clusters (line 4) requires $O(1)$ time if $r \neq t+1$ due to cached reuse, and $O(K^2 n \log n)$ otherwise, where $n$ is the number of examples from $t_1$ to $t+1$.
\item Incrementally updating the contingency table (line 5) is $O(K)$, as it involves placing a single new instance into an existing bin structure.
\item Computing the Chi-square statistic and deriving the p-value (line 6) is $O(2K)$.

\end{itemize}

In total, the time complexity is:
\[
O(K^2 n\log n + (t-t_1)(1 + 3K)).
\]

Since the for-loop can be parallelized via multi-threading, the complexity is effectively reduced to:
\[
O(K^2 n\log n + 3K),
\]
for the incremental component. This parallelization, combined with caching and incremental updates, enables efficient online PU-index-based drift detection with substantially reduced overhead while maintaining strong detection capabilities.

\section{Implementation details}\label{sec-details}

 Our experiment consists of four parts: comparison with classic drift detectors, comparison with SOTA methods, comparison with baselines on the CIFAR-10-CD dataset, and ablation study.

The experimental setup was established on a server running Ubuntu 18.04, which was outfitted with an NVIDIA A100 GPU and possessed 200GB of RAM.
The performance of drift detectors was evaluated across three classifiers: a Gaussian Naive Bayes classifier (GNB), a Hoeffding Tree Classifier (Very Fast Decision Tree, VFDT) \cite{vfdt}, and a Deep Neural Network (DNN) \cite{dnn}. GNB was implemented using the widely-used scikit-learn package \cite{scikit-learn}, VFDT via the River package \cite{river}, and the DNN, a three-layer neural network, through the PyTorch package.
The DNN consists of a multi-layer perceptron (MLP) followed by an output layer. For most datasets, the DNN comprises two hidden layers, each with 64 neurons and ReLU activation. The input dimension is dataset-dependent.
For the airline dataset, we employ three hidden layers (512, 256, and 64 neurons respectively), all using ReLU activation. The input dimension for this specific case is 679.
The final layer is a fully connected layer mapping from 64 neurons to the output dimension, which varies based on the datasets.

All experiments have been repeated 100 times with different random seeds and average accuracy is reported in this paper. The bins of the histogram, i.e., $k$, built by Adaptive PU-index bucketing are automatically determined by the Ei-kMeans algorithm. We initialize it as 5.
We first introduce the experiment of classic drift detectors. For each time step, classifiers receive a data chunk for prediction and then we feed the error rate into the drift detector for drift detection and adaptation. 
For a fair comparison, GNB and VFDT were adapted by retraining on the current chunk, while DNN adaptation involved resetting the last layer. The evaluation was conducted using two systems: Incremental way (test-and-train) and Train Once Until Alarm way (training the classifier only upon new concept detection or initialization). 
Moreover, the Adam optimizer was employed with an initial learning rate of 0.01 for the training of DDN, and the training involved 100 epochs. 
The setting of the ablation study is the same as above.

For comparison with SOTA methods, considering that all SOTA approaches, except MCDD, are based on ensemble learning, such as AMF and IWE using 10 classifiers, and others allowing up to 10,000 base classifiers, we ensured fairness by also using an ensemble approach for PUDD and MCDD. 
The classifier of our method comprises 5 DNNs as base classifiers. The final prediction is derived through soft voting based on the individual predictions of these classifiers. Additionally, we combine the prediction uncertainty from all 5 base classifiers for effective drift detection.
We also use the Adam optimizer with a learning rate of 0.01 for this experiment but with an increased epoch number of 100. The adaptation approach of this experiment is resetting the classifier when a drift is detected. The increased training epoch number and the adaptation approach of the resetting model are aimed at enhancing the learning capacity and fine-tuning the performance of the Deep Neural Networks within our ensemble framework.

For the proposed CIFAR-10-CD dataset, which is an image dataset with concept drift, we utilized ResNet-18 \cite{resnet} as the classifier model. The optimization was carried out using the Stochastic Gradient Descent (SGD) optimizer with a learning rate set to 0.01. We only evaluate incrementally for CIFAR-10-CD datasets due to their high learning difficulty and the training epoch is set as 5.

\section{Datasets introduction}
\begin{table}[tb]
    \centering
    \resizebox{\linewidth}{!}{%
    \begin{tabular}{@{\extracolsep{-5pt}} l|lllllll}
    \toprule
    \centering
        Datasets & \textit{Sea-N} & \textit{Sine} & \textit{Mixed} & \textit{Elec2} & \textit{Airline} & \textit{Powersupply} & \textit{CIFAR-10-CD} \\ 
        \midrule
        \# of feature & 3 & 2 & 4 & 8 & 679 & 4 & 32x32 \\ 
        \# of class & 2 & 2 & 2 & 2 & 2 & 2 & 10 \\ 
        \# of dataset & 100k & 100k & 100k & 45k & 58k & 29k & 50k \\ 
        chunk size & 1000 & 1000 & 1000 & 1000 & 1000 & 1000 & 100 \\
        \bottomrule
        
    \end{tabular}
    }
    \caption{Statistics of datasets.}\label{tab:dataset}
\end{table}

\begin{enumerate}
    \item Elec2\cite{elec}: A real-world dataset from the Australian electricity market with 45,312 instances (only the first 45,000 were used and split into 45 chunks). It includes features like electricity demand and a binary class label for price direction. 
    \item Airline\cite{airline}: This real-world dataset includes 581,0462 (we only use the first 58,000 instances) flight schedule records with binary labels for flight delays, containing concept drift due to changes in days, and times. The dataset is split into 58 chunks. The dataset is one-hot encoded and the dimension increases to 679 following \cite{driftsurf}. The dataset can be obtained through the official GitHub repository of the work. 
    \item PowerSupply\cite{powersupply}: A real-world dataset of 29,928 (only the first 29,000 instances are used and split into 29 chunks) hourly power supply records, exhibiting concept drift potentially due to seasonal changes or differences between weekdays and weekends.
    \item SEA\cite{moa}: A synthetic dataset from MOA \cite{moa} with 100,000 data points divided into 100 chunks, demonstrating concept drift by changing thresholds of classification function and including noise variants. We alter the threshold every tenth chunk.
    \item SINE\cite{ddm}: Characterized by two attributes, this dataset involves abrupt concept drift with 100,000 samples (divided into 100 chunks), where labels are determined by a sine function and altered every tenth chunk.
    \item Mixed\cite{ddm}: Designed for studying abrupt concept drift, this dataset comprises 100,000 boolean and numeric examples. We alter its classification function every ten chunks to simulate concept drift.
    \item CIFAR-10-CD: A variant of the CIFAR-10 image dataset, modified to simulate concept drift in user interests. The 50,000 images are divided into 100 chunks. The labels of interests change according to the Markov process.
\end{enumerate}

\section{Baselines introduction}
\begin{enumerate}
    \item DDM\cite{ddm}: This method monitors the online error rate of a learning algorithm and declares a new context when the error is higher than a predefined warning and drift level.
    \item EDDM\cite{EDDM}: EDDM detects concept drift, especially in slow and gradual changes, by measuring the distance between classification errors instead of the error rate.
    \item HDDM-A/W\cite{HDDM}: HDDM-A focuses on abrupt changes using a moving average, while HDDM-W detects gradual changes with weighted moving averages, both maintaining the efficiency of the original HDDM.
    \item KSWIN\cite{kswin}: Kolmogorov-Smirnov Windowing uses a sliding window and the Kolmogorov-Smirnov test to detect distribution changes in data streams.
    \item  PH\cite{ph}: This method enhances the Page Hinkley Test by using a data-dependent second-order intrinsic mode function for more robust change detection without the need for manual parameter tuning.
    \item ADWIN\cite{adwin}: ADWIN dynamically manages variable-sized data windows for concept drift in data streams, automatically adjusting the window size based on detected changes.
    \item IWE\cite{iwe}: This method adapts to new concepts by incrementally adjusting the weights of historical classifiers, using a variable-size window for rapid adaptation to changing data patterns.
    \item AMF\cite{amf}: An online variant of the Random Forest algorithm, AMF uses Mondrian Forests for real-time learning and adaptation. It prunes the decision trees to adjust to data pattern changes.
    \item ADLTER\cite{adlter}: It introduces an Adaptive Iterations method for Gradient Boosting Decision Tree. The method dynamically adjusts the number of iterations in response to drift severity. It also retrains classifiers when necessary.
    \item NS\cite{ns}: This method uses a Loss Improvement Ratio to evaluate GBDT weak learners. It proposed both Naive and Statistical Pruning strategies to adjust the model by pruning less effective learners in response to concept drift.
    \item MCDD \cite{mcdd}: A novel concept drift detection method that uses maximum concept discrepancy and contrastive learning of concept embeddings to adaptively identify various forms of concept drift in high-dimensional data streams.

\end{enumerate}


\begin{table}[h]
\setlength{\tabcolsep}{1mm}
\begin{tabular}{cl|cc|cc}
\hline
\multicolumn{2}{c|}{} & \multicolumn{2}{c|}{Incremental Training} & \multicolumn{2}{c}{The other case}\\ \hline
\multicolumn{1}{c|}{Classifier} & ddm\ name & sea10-I & sea20-I & sea10-O & sea20-O \\ \hline
\multicolumn{1}{c|}{\multirow{10}{*}{DNN}}& ADWIN & 87.89& \textbf{\underline{78.15}} & 84.85& 75.04\\
\multicolumn{1}{c|}{} & DDM& \textbf{\underline{88.04}} & \textbf{78.12} & 83.88& 74.30\\
\multicolumn{1}{c|}{} & EDDM& 83.63& 77.32& 77.47& 72.28\\
\multicolumn{1}{c|}{} & HDDM-A& 87.90& 77.81& 85.58& 75.74\\
\multicolumn{1}{c|}{} & HDDM-W& \textbf{87.92} & 77.61& 84.72& \textbf{76.04} \\
\multicolumn{1}{c|}{} & KSWIN & \textbf{87.97} & \textbf{78.09} & 82.31& 74.54\\
\multicolumn{1}{c|}{} & PH & 87.79& 77.92& \textbf{85.92} & 75.85\\ \cline{2-6} 
\multicolumn{1}{c|}{} & PUDD-1& 87.86& 77.50& 85.91& 76.00\\
\multicolumn{1}{c|}{} & PUDD-3& 87.80& 77.33& \textbf{86.02} & \textbf{76.35} \\
\multicolumn{1}{c|}{} & PUDD-5& 87.80& 77.38& \textbf{\underline{86.24}} & \textbf{\underline{76.38}} \\ \hline
\multicolumn{1}{c|}{\multirow{10}{*}{GNB}}& ADWIN & 86.24& 76.90& 85.43& 75.63\\
\multicolumn{1}{c|}{} & DDM& 85.47& 76.26& 84.58& 74.72\\
\multicolumn{1}{c|}{} & EDDM& 86.26& 76.67& \textbf{85.78} & \textbf{75.91} \\
\multicolumn{1}{c|}{} & HDDM-A& 86.63& 77.22& 85.82& 76.25\\
\multicolumn{1}{c|}{} & HDDM-W& 85.70& 77.44& 85.05& 76.70\\
\multicolumn{1}{c|}{} & KSWIN & 84.34& 76.19& 82.84& 74.73\\
\multicolumn{1}{c|}{} & PH & 86.89& 77.57& 86.18& 76.67\\ \cline{2-6} 
\multicolumn{1}{c|}{} & PUDD-1& \textbf{87.24} & \textbf{\underline{78.08}} & 86.53& \textbf{76.89} \\
\multicolumn{1}{c|}{} & PUDD-3& \textbf{\underline{87.28}} & \textbf{77.86} & \textbf{\underline{86.61}} & \textbf{\underline{77.02}} \\
\multicolumn{1}{c|}{} & PUDD-5& \textbf{87.18} & \textbf{77.68} & \textbf{86.55} & 76.67\\ \hline
\multicolumn{1}{c|}{\multirow{10}{*}{VFDT}} & ADWIN & 85.50& 76.05& 85.58& 75.86\\
\multicolumn{1}{c|}{} & DDM& 84.89& 75.87& 84.62& 74.96\\
\multicolumn{1}{c|}{} & EDDM& 86.13& 76.59& 85.97& 76.48\\
\multicolumn{1}{c|}{} & HDDM-A& 85.62& 76.22& 85.83& 76.41\\
\multicolumn{1}{c|}{} & HDDM-W& 85.34& 76.44& 85.15& \textbf{76.73} \\
\multicolumn{1}{c|}{} & KSWIN & 84.53& 75.86& 82.64& 74.66\\
\multicolumn{1}{c|}{} & PH & 85.67& 76.41& 86.15& 76.57\\ \cline{2-6} 
\multicolumn{1}{c|}{} & PUDD-1& \textbf{\underline{86.35}} & \textbf{\underline{77.02}} & \textbf{86.30} & \textbf{76.85} \\
\multicolumn{1}{c|}{} & PUDD-3& \textbf{86.25} & \textbf{76.90} & \textbf{\underline{86.57}} & \textbf{\underline{76.86}} \\
\multicolumn{1}{c|}{} & PUDD-5& \textbf{86.19} & \textbf{76.83} & \textbf{86.19} & 76.29\\ \hline
\end{tabular}
\caption{Comparative analysis against classic drift detectors across 2 synthetic datasets. The top 3 results are highlighted in bold and the top 1 results are in both bold and underlined. PUDD-$x$ represents the threshold set as $10^{-x}$ for our method. }
\end{table}

\begin{table}[]
\centering
\begin{tabular}{l|ll}
\hline
\textbf{ddm name} & \textbf{sea10-I}                        & \textbf{sea20-I}                        \\ \hline
AMF               & 83.70                                   & 73.41                                   \\
IWE               & 84.73                                   & 74.33                                   \\
NS                & 84.39                                   & 76.00                                   \\
ADLTER            & 85.89                                   & 76.48                                   \\
MCD-DD            & 87.22                                   & 77.25                                   \\ \hline
PUDD-1            & \textbf{87.72}                          & \textbf{76.93}                          \\
PUDD-3            & \textbf{87.67}                          & \textbf{77.22}                          \\
PUDD-5            & \textbf{\underline{87.74}} & \textbf{\underline{77.32}} \\ \hline
\end{tabular}
\caption{Comparison with SOTA methods. The table shows that our methods PUDD in achieved top-1 in 2 out of 2 datasets, implying the effectiveness of PUDD compared with SOTA methods.}
\end{table}

\section{Observations of PUDD performance}
    \textbf{Observation 5: increment learning is generally better than training only once until alarm way} from Table 1. This observation aligns with the PAC theory \cite{pac}, which posits that the error of the classifier is inversely proportional to the number of training data. The reason is that this training approach allows classifiers to learn continuously and adapt to new data. This way, the classifiers produce prediction uncertainty that better reflects the current concept, leading to more accurate drift detection. The only exception is that DNN shows better performance on mixed datasets under the setting of trained only once until alarm way, e.g. when the threshold is set as $10^{-5}$, the test accuracy is 82.81\% for incremental training vs 84.90\% for the counterpart. The potential reason is the dataset contains both boolean and numeric datasets, which requires data preprocessing for DNN to conduct continual learning. However, our experiments aim to compare PUDD with baselines to verify its effectiveness, thus we did not specifically design data preprocessing for the dataset. Moreover, the overall experiment result confirms the advantage of incremental learning.

\textbf{Observation 6: our method demonstrates higher enhanced performance with DNNs compared to classic machine learning classifiers, e.g. GNB and VFDT}. For example, under the incremental training setting, our approach with DNNs surpasses the other two classifiers in 6 out of 8 datasets. In comparison, ADWIN combined with DNN only exceeds the performance of the other classifiers in 4 datasets. This observation suggests that the PU-index generated by DNNs provides more detailed insights than those from GNB and VFDT. Such granularity acts as a more informative indicator for drift detection, and our method, PUDD, capitalizes on this feature more effectively than the other baseline approaches.

\section{Critical difference diagrams of our experiments}

\begin{figure}[h]
    \centering
    \begin{subfigure}{\linewidth}
        \centering
        \includegraphics[width=\linewidth]{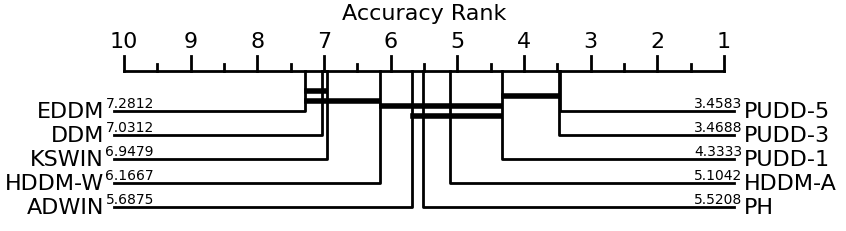}
        \label{fig:sub1}
    \end{subfigure}

    \caption{Critical difference diagrams of the rank of PUDD and classic drift detectors. The figure shows that the PUDD with threshold $10^{-3}$ and $10^{-5}$  outperform all classic drift detectors and the result is statistically significant.}
\end{figure}

\begin{figure}[h]
    \centering
    \begin{subfigure}{\linewidth}
        \centering
        \includegraphics[width=\linewidth]{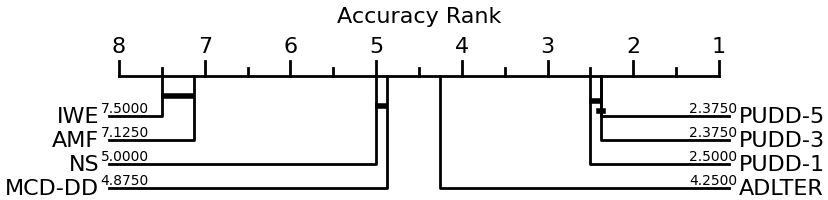}
        \label{fig:sub2}
    \end{subfigure}

    \caption{Critical difference diagrams of the rank of PUDD and SOTA methods. The figure shows that the PUDD outperforms all SOTA methods and the result is statistically significant.}
\end{figure}

\begin{figure}[h]
    \centering
    \begin{subfigure}{\linewidth}
        \centering
        \includegraphics[width=\linewidth]{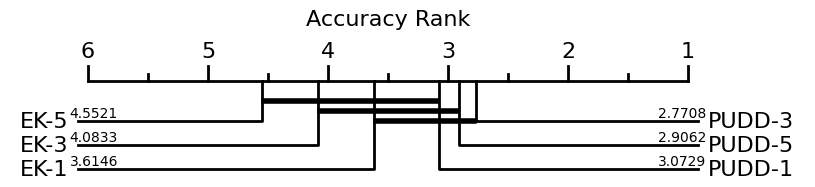}
        \label{fig:sub3}
    \end{subfigure}

    \caption{Critical difference diagrams of the rank of PUDD (based on Adaptive PU-index bucketing) and EK (original Ei-kMeans partitioning). The suffix of each method indicates the logarithm of the threshold value. The figure shows that the improvement of the Adaptive PU-index Bucketing with threshold $10^{-3}$ and $10^{-5}$ are statistically significant.}

    \label{fig:twopart}
\end{figure}

\end{document}